\documentclass{article}

\usepackage{arxiv}
\date{}

\usepackage[utf8]{inputenc} 
\usepackage[T1]{fontenc}    
\usepackage{hyperref}       
\usepackage{url}            
\usepackage{booktabs}       
\usepackage{amsfonts}       
\usepackage{nicefrac}       
\usepackage{microtype}      

\usepackage{graphicx}
\usepackage{subfigure}

\usepackage{amsthm}
\usepackage{amsmath}
\usepackage{amssymb}
\usepackage{bm}
\usepackage{mathrsfs}

\usepackage{algorithmic}
\usepackage{algorithm}

\usepackage{color}

\usepackage{url}

\usepackage{supertabular}

\newtheorem{thm}{Theorem}

\theoremstyle{definition}

\theoremstyle{remark}

\numberwithin{thm}{section}

\DeclareMathAlphabet{\mathsfsl}{OT1}{cmss}{m}{sl}

\renewcommand{\phi}{\varphi}

\newcommand{\argmin}{\operatorname*{arg\; min}}

\newcommand{\Expect}{\operatorname{\mathbb{E}}}

\newcommand{\supp}[1]{\operatorname{supp}(#1)}

\newcommand{\bx}{\boldsymbol{x}}

\renewcommand{\algorithmicrequire}{\textbf{Input:}}
\renewcommand{\algorithmicensure}{\textbf{Output:}}

\def\reals{\mathbb{R}}
\def\bx{\boldsymbol{x}}

\def\b0{\mathbf{0}}

\def\bbP{\mathbb{P}}

\def\calB{\mathcal{B}}

\def\calH{\mathcal{H}}
\def\calR{\mathcal{R}}

\usepackage{algorithmic}
\usepackage{algorithm}

\usepackage{tikz}
\usepackage{amsmath}

\usepackage{amsmath}
\usepackage{amssymb}
\usepackage{booktabs}
\usepackage{multirow}
\usepackage{verbatim}

\usepackage{comment}
\usepackage{amssymb}
\usepackage{pifont}
\usepackage{graphicx}
\usepackage{indentfirst}  
\usepackage{filecontents}
\usepackage{url}

\title{Exploring Adversarial Examples for Efficient Active Learning in Machine Learning Classifiers}

\author{%
   Honggang Yu\thanks{Equal Contribution}\\
   University of Florida \\
  \texttt{honggang.yu@ufl.edu} \\
  \And
  Shihfeng Zeng$^*$ \\
  National Cheng Kung University\\
  \texttt{p76081491@gs.ncku.edu.tw}\\
  \And  
  Teng Zhang \\
  University of Central Florida\\
  \texttt{teng.zhang@ucf.edu} \\
   \And
   Ing-Chao Lin \\
   National Cheng Kung University \\
   \texttt{iclin@mail.ncku.edu.tw} \\
   \And
   Yier Jin \\
   University of Florida \\
   \texttt{yier.jin@ece.ufl.edu} \\
}

\begin{document}

\maketitle
\begin{abstract}
Machine learning researchers have long noticed the phenomenon that the model training process will be more effective and efficient when the training samples are densely sampled around the underlying decision boundary. While this observation has already been widely applied in a range of machine learning security techniques, it lacks theoretical analyses of the correctness of the observation. To address this challenge, we first add particular perturbation to original training examples using adversarial attack methods so that the generated examples could lie approximately on the decision boundary of the ML classifiers. We then investigate the connections between active learning and these particular training examples. Through analyzing various representative classifiers such as $k$-NN classifiers, kernel methods as well as deep neural networks, we establish a theoretical foundation for the observation. As a result, our theoretical proofs provide support to more efficient active learning methods with the help of adversarial examples, contrary to previous works where adversarial examples are often used as destructive solutions. Experimental results show that the established theoretical foundation will guide better active learning strategies based on adversarial examples.
\end{abstract}

\section{Introduction}
Machine learning (ML) models have become a core component of many real-world applications, ranging across object recognition \cite{Lee_2018_CVPR, Feng_2019_CVPR, Ku_2019_CVPR}, natural language processing \cite{10.1007/978-3-642-14706-7_20} and autonomous vehicles \cite{Conneau16, silva2018a}. However, creating state-of-the-art ML models often requires abundant annotated training examples, which may be time-consuming to label or even impractical to obtain at large scale.
To remedy this situation, active learning algorithms are gaining popularity recently thanks to their ability to attain high accuracy with low labeling cost. Many existing works have proved that active learning algorithm based ML models use fewer training instances (i.e., ``useful'' examples) in various security-crucial scenarios. In particular, machine learning researchers have noticed the phenomenon that the model training process will be more effective and efficient with less computation complexity when the training samples are selected using active learning strategies. Specifically, an adversary may be able to make use of active learning to iteratively select adversarial examples (i.e., ``informative'' examples) around the underlying decision boundary, 
while maximizing the performance of retrained machine learning classifiers \cite{247638, 251526, CloudLeak2020}.  
However, the fundamental cause of this phenomenon is not thoroughly studied. To address this challenge, in this paper we analyze various widely-used ML models, including $k$-nearest neighbors ($k$-NN), kernel regression and deep neural networks, and explore the connection between active learning and adversarial examples. The key idea is that, since the adversarial examples sampled by the active learning strategy (i.e., margin-based uncertainty sampling strategy) lying approximately on the decision boundary of the victim ML classifier, a user can greatly reduce the labeling effort when generating the synthetic data set for retraining the shadow model. The inherent connections between active learning and adversarial examples can help improve the efficiency of queries to target ML classifiers.

Specifically, we will evaluate the connections on ML classifiers trained on four popular datasets (e.g., Halfmoon, Abalone, MNIST and CIFAR10). As we will show in the paper, the experimental results demonstrate that adversarial examples are more efficient than random examples while training the ML classifiers. Besides the more efficient ML model construction, the finding in this paper may also be applied in the ML security domain. For example, an adversary may utilize the established theoretical foundation to design more powerful security and privacy attack strategies for adversarial attacks, model extraction, member inference attacks, etc.

The rest of the paper is organized as follows. Section \ref{sec:related} presents some background knowledge and related work. Section \ref{sec:analysis} analyzes the performance of ML classifiers with various parameters. Section \ref{sec:adversarial_al} investigates adversarial active learning for these classifiers. Section \ref{sec:experiment} discusses the  basic  setup  of  our method and experimental results. Conclusion is presented in Section \ref{sec:conclusion}.





\section{Background and Related Work}\label{sec:related}

\subsection{Adversarial Attacks}

Machine learning models are gaining popularity in various security-crucial scenarios, including image recognition, natural language process and autonomous vehicles. Despite being popular, ML models have been shown to be vulnerable to adversarial attacks \cite{papernot16a,Sharif16, Papernot16, Szegedy13,papernot15,Carlini17, Madry2017TowardsDL, Biggio18, Sethi18, Ji18} that generate adversarial examples by adding particular perturbations to the original inputs. In general, these carefully crafted examples are often imperceptible to humans but can easily force ML models to output incorrect results. Previous works on adversarial attacks against ML models can be divided into the following two categories: 1) white-box attacks; 2) black-box attacks. For the white-box attacks, adversaries assume that they can obtain the inner information from target/victim ML models, including the exact training dataset, architectures and parameters. For example, Szegedy \textit{et al.} \cite{Szegedy13} presented the first optimization-based attack method named L-BFGS to craft the adversarial examples that can easily evade the target/victim ML models. Many following works have been proposed in an effort to improve the performance of adversarial attacks \cite{goodfellow2014explaining, papernot15, Carlini17}. For the black-box attacks, adversaries assume that they have no access to the details of target/victim ML models but can acquire the output results (e.g., labels and confidence score) while providing random inputs. These attacks rely on either the substitute models or the gradient approximation for crafting the adversarial examples that can easily fool the black-box ML models into predicting incorrect results. Different from previous works, in this paper we use these adversarial attacks to craft those particular examples around the decision boundary of the ML classifiers and generate the synthetic datset fro training the classifiers. 


\subsection{Active Learning}
Active learning aims to select a “useful” or “informational” subset of unlabeled data set (i.e., image pool) for effectively training the classifier in the target domain by exploiting special  sample  strategies such as random  strategy, certainty strategy  and uncertainty  strategy \cite{NIPS2007_3252,David1996,Settles2008,Tong2002}.
As a result, this selection process would greatly reduce  the  label  efforts  of human experts while simultaneously maximizing performance of the classifier. However, traditional active learning algorithms usually suffer from a serious problem that the selected examples are often overly similar to each other, leading to the poor classification performance of machine learning models. To remedy this situation, Yu \textit{et al} \cite{CloudLeak2020} proposed to 
make use of the adversarial active learning strategy for selecting the useful examples from the image pool. Though their methods show that the model training process is be more effective and efficient when the training samples are densely sampled around the underlying decision boundary, theoretical analyses of the correctness of these observations are rarely discussed. In this paper, we improve their work by investigating the connections between active learning and adversarial examples and further establish a theoretical foundation for these observations. Consequently, ML researchers will have theoretical proofs before exploiting these observations in developing state-of-the-art adversarial attack strategies.

\subsection{Classifiers and their theoretical properties}

Given a data point $\bx$ from the feature space $\reals^d$, we aim to predict its label $y\in\reals$. We treat feature vector $\bx$ as being sampled from a marginal distribution
$\mu$ in $\reals^d$. After sampling feature
vector $\bx$, its label $y$ is sampled from a conditional distribution $\bbP_{y|\bx}$, and we denote $D$ as the distribution of joint variables $(\bx,y)$. 

To assist us in making predictions, we assume we have access to $n$
training data pairs $\{(\bx_i,y_i)\}_{i=1}^n$ from $\reals^d\times\reals$ that are
i.i.d. sampled as follows: first, each $\bx_i$ is sampled from a distribution $\mu\in\reals^d$. Then, its label $y_i$ is sampled from the underlying conditional distribution $\bbP_{y|\bx}$, and we denote $D'$ as the distribution of each $(\bx_i,y_i)$. For the prediction problem of regression, $y$ can be any values in $\reals$, and for the prediction problem of binary classification, $y\in\{0,1\}$, and we let $\eta(\bx)=\Expect_{D}(y|\bx)$.

For the classification problem, the decision boundary is defined as the set of the points that has equal probabilities to be in either classes, i.e, the set $\calB=\{\bx\in\reals^d: \eta(\bx)=0.5\}$. It has been observed that in many applications, it is advisory to make the distribution of the training samples $\{\bx_i\}_{i=1}^n$ concentrated around the decision boundary $\calB$. However, this observation has not been theoretically analyzed yet, and this work will establishes such guarantees for some popular algorithms, such as $k$-nearest neighbors and kernel methods, and regression in reproducing Hilbert kernel spaces.

The theoretical properties of these classifier have been well studied in literature. For example, the $k$-nearest neighbors has been studied in \cite{10.5555/2969033.2969210,biau2015lectures,chen2018explaining}, the kernel methods has been studied in \cite{10.2307/3315046,COLLOMB198677,10.2307/20142515}, and the method based on reproducing kernel Hilbert space has been analyzed in \cite{SMALE2005285,Caponnetto2007,mendelson2010,liu2020nonasymptotic}. These works usually assume the asymptotic setting that $n\rightarrow\infty$. 
However, all existing asymptotic results assumed that the training samples $\{\bx_i\}_{i=1}^n$ have the same distribution as the test sample $\bx$, which does not explain the observation that when the training samples are concentrated around the decision boundary, the estimators perform better. It leaves the basic question that has motivated the present paper: Is
it possible to explain this fact for learning methods such as $k$-nn and kernel methods?


\section{Model Analysis}\label{sec:analysis}

In this section, we will investigate the dependence of the convergence of a few commonly used classifiers on the distribution of the training set $\{\bx_i\}_{i=1}^n$, both asymptotically and non-asymptotically. Recall that we use $\mu$ to denote the distribution of training samples, i.e., $\{\bx_i\}_{i=1}^n$ are i.i.d. sampled from $\mu$, we will analyze the performance of estimators for various choices of $\mu$ and investigate the region where the estimator is correct:
\[
\calR_{\{\bx_i,y_i\}_{i=1}^n}=\{\bx\in\reals^d: \hat{g}_{\{\bx_i,y_i\}_{i=1}^n}(\bx)=g(\bx)\}.
\]
Here $\hat{g}_{\{\bx_i,y_i\}_{i=1}^n}$ is the classifier obtained from the training samples $\{\bx_i,y_i\}_{i=1}^n$, and $g(\bx)$ is the optimal Bayesian classifier. Clearly, the performance of the classifier can be measured by the area of $\calR_{\{\bx_i,y_i\}_{i=1}^n}$: larger area means better performance.

To show that the classifiers perform better when $\mu$ is concentrated around the decision boundary $\calB$, we investigate a class of distributions $\mu$ defined as follows:
\[\mu(\bx)=C_{\beta}/|\eta(\bx)-1/2|^{\beta},\] where $\beta\geq 0$ and $C_{\beta}$ is a normalization constant. Recall that $\calB=\{\bx: \eta(\bx)=1/2\}$, this distribution is more concentrated around the decision boundary as $\beta$ gets larger. In particular, when $\beta=0$, the distribution is uniform and when $\beta=\infty$, the distribution is supported at $\calB$. In the remaining of this section, we will show that for a list of standard learning methods, the asymptotic performance when $\beta>0$ is better the performance when $\beta=0$, i.e., the algorithm performs better when the distribution $\mu$ is more concentrated around the decision boundary.

For the convenience of analysis, we make two natural assumptions: 
\begin{itemize}
    \item $\mu$ has a compact support in $\reals^d$.
  \item The decision boundary $\calB$ is a smooth $d-1$-dimensional manifold.
\end{itemize}
We remark that the second assumption assumes that the decision boundary is well-behaved. For the special case that $\calB$ is a $d-1$-dimensional hyperplane, the distribution $\mu(\bx)=C_{\beta}/|\eta(\bx)-1/2|^{\beta}$ is only well-defined for some $\beta>0$.

The main intuition in the proof can be summarized as follows: to make the classifier correct at location $\bx$, the estimator $\hat{\eta}(\bx)$ should satisfies $|\hat{\eta}(\bx)-{\eta}(\bx)|<|{\eta}(\bx)-1/2|$. As a result, the estimators are allowed to make larger errors at $\bx$ with larger $|{\eta}(\bx)-1/2|$, i.e., when $\bx$ is further away from the decision boundary. That is, few sampling is required at such locations. In comparison, for $\bx$ near the decision boundary, we need a denser sampling around $\bx$ to obtain a higher accuracy.


\subsection{$k$-nearest neighbors}
For the regression problem, the $k$-NN regression estimate is given by
\begin{equation}\label{eq:knn}
     \hat{\eta}_{knn}(\bx)=\frac{1}{k}\sum_{i=1}^k y_{(i)}(\bx),
\end{equation}
where $(\bx_{(1)}(\bx),y_{(1)}(\bx)), (\bx_{(2)}(\bx),y_{(2)}(\bx))$, $\cdots$, $(\bx_{(n)}(\bx),y_{(n)}(\bx))$ is a reordering of $(\bx_1(\bx),y_{(1)}(\bx))$ according to increasing values of $\|\bx_i-\bx\|$, and the $k$-NN classification estimator is given by
\begin{equation}\label{eq:knn2}
     \hat{g}_{knn}(\bx)=I\left(\hat{\eta}(\bx)>\frac{1}{2}\right),
\end{equation}

There has been many works about the convergence of the $k$-nn estimator. For example, Biau and Devroye \cite{biau2015lectures} have proved the convergence under the notion of uniform consistency, that is, $\sup_{\bx}|\hat{\eta}_{k}(\bx)-\eta_{k}(\bx)|\rightarrow 0$ almost surely.  In addition, the nonasymptotic convergence is investigated in \cite{10.5555/2969033.2969210} and reviewed in \cite{chen2018explaining}. Compared to their results, we analyze the nonasymptotic performance for generic training sets $\{\bx_i,y_i\}_{i=1}^n$, and the main result is as follows:
\begin{thm}\label{thm:knn}
Assume that $\eta$ is $\alpha$-Holder continuous, i.e. $|\eta(\bx)-\eta(\bx')|\leq L\|\bx-\bx'\|^\alpha$, then with probability $1-2\delta$, the $k$-nearest neighbor classifier $\hat{g}_{knn}$ is correct for the set
\begin{align}\label{eq:knn_set}
&\Big\{\bx: |\eta(\bx)-\frac{1}{2}|-L(r_{k/n+\delta_p}(\bx))^\alpha>C_0 \Big\},
\end{align}
where $\delta_p=\frac{C_0}{n}\Big(d\log n+\log(1/\delta)+\sqrt{k(d\log n+\log(1/\delta))}\Big)$, $r_p(\bx_0)=\inf\{r|\mu(\bx: \|\bx-\bx_0\|\leq r)\geq p\}$, and  $C_0=\sqrt{\frac{1}{2k}\Big(2\log 2+ \log(n^{d+1}+1)-\log {\delta}\Big)}$.
\end{thm}
Note that the constant $C_0$ is independent of $\mu$, the dependence of the set  in \eqref{eq:knn_set} depends on $\mu$ through $r_{k/n+\delta_p}(\bx)$, which is in the order of $O((\frac{k/n+\delta_p}{\mu(\bx)})^{1/d})$ as $n\rightarrow\infty$. Applying Theorem~\ref{thm:knn} to $\mu(\bx)=C/|\eta(\bx)-1/2|^{\beta}$, and assuming that $k>O(d \log n)$, then $r_{k/n+\delta_p}(\bx)$ is in the order of $O\Big((k/n)^{1/d}|\eta(\bx)-1/2|^{\beta/d}\Big).$ Since this is an increasing number as $\beta$ increases, the set of correct classification in \eqref{eq:knn_set} is also larger as $\beta$ increases, i.e., when the distribution of the training set is more concentrated around the decision boundary. 

\begin{proof}[Proof of Theorem~\ref{thm:knn}]
To investigate a non-asymptotic result property of the $k$-nn estimator, we apply a recent result (see appendix) as follows: 
Assuming that $|y|\leq M_y$ holds almost surely for all $(\bx,y)$ and $\eta$ is $\alpha$-Holder continuous, i.e. $|\eta(\bx)-\eta(\bx')|\leq L\|\bx-\bx'\|^\alpha$, then  with probability $1-2\delta$, the $k$-NN estimator satisfies that \begin{align}\label{eq:regression}
&|\hat{\eta}(\bx_0)-\eta(\bx_0)| <  L(r_{k/n+\delta_p}(\bx_0))^\alpha\\&+\sqrt{\frac{1}{2k}\Big(2\log 2+ \log(n^{d+1}+1)-\log {\delta}\Big)},\,\,\text{for all $\bx_0\in\reals^d$}.\nonumber
\end{align}
Theorem~\ref{thm:knn} follows from this observation with some rearrangements.
\end{proof}


\subsection{Kernel Regression}

The Nadaraya–Watson kernel estimator is given by the locally weighted average
\[
\hat{\eta}_{NW}(\bx)=\frac{\sum_{i=1}^n K(\frac{\bx-\bx_i}{h})y_i}{\sum_{i=1}^n K(\frac{\bx-\bx_i}{h})},
\]
where $K: \reals^d\rightarrow\reals$ is a kernel function which is multivariate symmetric and nonnegative, $h$ is a parameter for tuning the width of the kernel, and the classifier associated with $\hat{\eta}_{NW}$ is define (similar to \eqref{eq:knn2}) by $\hat{g}_{NW}(\bx)=I\left(\hat{\eta}(\bx)>\frac{1}{2}\right)$. If the $K$ is the indicator function $K(\bx)=I(|\bx|<1)$, then the Nadaraya–Watson estimator reduces to the fixed-radius near neighbor regression. For this estimator, we have the following upper bound on estimation error:
\begin{thm}\label{thm:kernel}
Assuming that the kernel function $K$ is  Lipschitz continuous with parameter $L_{K}$ and supported in a unit ball in $\reals^d$, $\eta$ is Lipschitz continuous with parameyer $L_{\eta}$ then with  probability that goes to $1$ as $n\rightarrow\infty$, for all $\bx$,
\[
|\hat{\eta}_{NW}(\bx)-\eta(\bx)|\leq  hL_{\mu}+\frac{2\log^2 n}{Cnh^d\mu(\bx)-\log^2 n}.
\]
As a result, the classifier $\hat{g}_{NW}$ is correct for the set
\[
\Big\{\bx: |\eta(\bx)-\frac{1}{2}|\geq  hL_{\mu}+\frac{2\log^2 n}{Cnh^d\mu(\bx)-\log^2 n}\Big\}.
\]
\end{thm}
The theorem shows that the set of correct classification is larger when $\mu$ is more concentrated around the decision boundary. In particular, when $\mu(\bx)$ is in the order of $1/|\eta(\bx)-\frac{1}{2}|^{\beta}$, then the classifier is correct for that set $\{|\eta(\bx)-\frac{1}{2}|\geq O(\max(h,\frac{\log^2n}{nh^d}^{\frac{1}{1+\beta}})\}$, which has a minimum value in the order of $O(\frac{n}{\log^2n}^{-1/(d+1+\beta)})$, which is a larger value asymptotically when $\beta$ is larger, i.e., when the distribution is more concentrated around the decision boundary.


\begin{proof}[Proof of Theorem~\ref{thm:kernel}]
To investigate the convergence of $\hat{\eta}_{NW}(\bx)-\eta(\bx)$ for generic $\mu$ that changed with $n$, we will first investigate 
\[
\tilde{\eta}_{NW}(\bx)=\frac{\Expect_{\bx_0\sim\mu} K(\frac{\bx-\bx_0}{h})\eta(\bx_0)}{\Expect_{\bx_0\sim\mu} K(\frac{\bx-\bx_0}{h})}.
\]
Since $K$ has support in $B(0,1)$, $\min_{\bx'\in B(\bx,h)}\eta(\bx)\leq \tilde{\eta}_{NW}(\bx)\leq \max_{\bx'\in B(\bx,h)}\eta(\bx)$. In addition, applying Hoeffding's inequality, for any fixed $\bx$,
\begin{align*}
&\Pr\left(\frac{1}{n}\sum_{i=1}^n K(\frac{\bx-\bx_i}{h})y_i-\Expect_{\bx_0\sim\mu} K(\bx-\bx_0)\eta(\bx_0)>t\right)\\\leq &2\exp(-2nt^2)
\end{align*}
and
\begin{align*}
&
\Pr\left(\frac{1}{n}\sum_{i=1}^n K(\frac{\bx-\bx_i}{h})-\Expect_{\bx_0\sim\mu} K(\bx-\bx_0)>t\right)\leq 2\exp(-2nt^2).\end{align*}
Applying the $\epsilon$-net argument to $\supp\mu$, with probability $1-4C\exp(-2nt^2)(\frac{L}{L})^d$, we have that for all $\bx\in\supp\mu$,
\[
\frac{1}{n}\sum_{i=1}^n K(\frac{\bx-\bx_i}{h})y_i-\Expect_{\bx_0\sim\mu} K(\bx-\bx_0)\eta(\bx_0)<2t
\]
and
\[
\frac{1}{n}\sum_{i=1}^n K(\frac{\bx-\bx_i}{h})-\Expect_{\bx_0\sim\mu} K(\bx-\bx_0)<2t.
\]
Considering that $\Expect_{\bx_0\sim\mu} K(\frac{\bx-\bx_0}{h})$ is in the order of $h^d\mu(\bx)$ as $h\rightarrow 0$, with probability $1-4C\exp(-2nt^2)(\frac{L}{t})^d$,
\begin{equation}\label{eq:diff}
|\tilde{\eta}_{NW}(\bx)-\hat{\eta}_{NW}(\bx)|\leq \frac{4t}{Ch^d\mu(\bx)-2t}.
\end{equation}
Combining this result with $t=\log^2n/n$ and the estimation of $\tilde{\eta}_{NW}(\bx)-\eta(\bx)$, the theorem is proved.
\end{proof}
\subsection{Regression in reproducing kernel Hilbert space}
Another popular class of estimator is based on the reproducing kernel Hilbert space (RKHS). The understanding of this method is important for the understanding of deep neural networks \cite{NIPS2017_6968,JMLR:v20:17-621} it gives a normed functional space $\calH$ and attempts to solve the problem
\begin{equation}\label{eq:RKHS}
\hat{\eta}_{RKHS}=\argmin_{\eta\in \calH}\sum_{i=1}^n(y_i-\eta(\bx_i))^2+\lambda_n\|\eta\|_H^2.
\end{equation}
A natural classifier based on $\hat{\eta}_{RKHS}$ is $\hat{g}_{RKHS}=I(\hat{\eta}_{RKHS}>1/2)$. We make the following assumptions:
\begin{itemize}
\item Any function $f\in\calH$ is Lipschitz with constant $L\|f\|_{\calH}$.
\item For each $\bx$ and $h>0$, there exists a function $f_{\bx,h}$ such that $f_{\bx,h}$ is supported $B(\bx,h)$ with $\|f_{\bx,h}\|_{\calH}\leq C/h$ and $\int_{\bx\in\reals^d} f_{\bx,h}(\bx)=1$.
\end{itemize}
\begin{thm}\label{thm:RKHS}
With probability goes to $1$ as $n\rightarrow\infty$, for all $\bx\in\supp\mu$, we have \[
|\hat{\eta}_{RKHS}(\bx)- \eta(\bx)|\leq \frac{2t}{Cnh^d\mu(\bx)-t}+Lh+\frac{C}{h(Cnh^d-t)}\sqrt{\frac{n}{\lambda_n}}.\]
As a result, the classifier $\hat{g}_{RKHS}$ is correct if
\begin{equation}\label{eq:RKHS_classifier}
|\eta(\bx)-\frac{1}{2}|\geq \frac{2t}{Cnh^d\mu(\bx)-t}+Lh+\frac{C}{h(Cnh^d-t)}\sqrt{\frac{n}{\lambda_n}}.
\end{equation}
\end{thm}
Since the RHS of \eqref{eq:RKHS_classifier} depends on $\mu$ only by the first term, similar to the analysis after Theorem~\ref{thm:kernel}, the performance of the classifier improves as $\beta$ increases and the distribution is more concentrated around the decision boundary.

\begin{proof}[Proof of Theorem~\ref{thm:RKHS}]
By definition in \eqref{eq:RKHS}, if we denote the function of zeros with $0_f$, then  $\lambda_n\|\hat{\eta}_{RKHS}(\bx)\|_{\calH}^2\leq \sum_{i=1}^n(y_i-0_f(\bx_i))^2+\lambda_n\|0_f\|_{\calH}\leq n$. As a result, $\|\hat{\eta}_{RKHS}(\bx)\|_{\calH}\leq \sqrt{n/\lambda_n}$.

Consider the function $\tilde{\eta}=\hat{\eta}_{RKHS}-\epsilon f_{\bx,h}$. Then 
\begin{align*}
&0\geq (\sum_{i=1}^n(y_i-\hat{\eta}_{RKHS}(\bx_i))^2+\lambda_n\|\hat{\eta}_{RKHS}\|_H^2)
    \\-&(\sum_{i=1}^n(y_i-\tilde{\eta}(\bx_i))^2+\lambda_n\|\tilde{\eta}\|_H^2)\\ \geq &
2\epsilon\left(\sum_{i=1}^n(y_i-\hat{\eta}_{RKHS}(\bx_i))f_{\bx,h}(\bx_i)-\|f_{\bx,h}\|_{\calH}\|\hat{\eta}_{RKHS}\|_H\right)
\\\geq &2\epsilon\left(\sum_{i=1}^n(y_i-\hat{\eta}_{RKHS}(\bx)-L\|\bx_i-\bx\|)f_{\bx,h}(\bx_i)-\frac{C}{h}\sqrt{\frac{n}{\lambda_n}}\right)
\\\geq &2\epsilon\left(\sum_{\bx_i\in B(\bx,h)}(y_i-\hat{\eta}_{RKHS}(\bx)-Lh)f_{\bx,h}(\bx_i)-\frac{C}{h}\sqrt{\frac{n}{\lambda_n}}\right).
\end{align*}
That is,
\begin{align*}
&\sum_{\bx_i\in B(\bx,h)}(y_i-\hat{\eta}_{RKHS}(\bx))f_{\bx,h}(\bx_i)\\\leq& Lh\sum_{\bx_i\in B(\bx,h)}f_{\bx,h}(\bx_i)+\frac{C}{h}\sqrt{\frac{n}{\lambda_n}}
\end{align*}
and
\begin{align*}
&\hat{\eta}_{RKHS}(\bx)\\\geq& \frac{\sum_{\bx_i\in B(\bx,h)}f_{\bx,h}(\bx_i)y_i}{\sum_{\bx_i\in B(\bx,h)}f_{\bx,h}(\bx_i)}-Lh-\frac{C}{h\sum_{\bx_i\in B(\bx,h)}f_{\bx,h}(\bx_i)}\sqrt{\frac{n}{\lambda_n}}\\
\geq &\eta(\bx)-\frac{2t}{Cnh^d\mu(\bx)-t}-Lh-\frac{C}{h(Cnh^d-t)}\sqrt{\frac{n}{\lambda_n}}.
\end{align*}
Similarly, 
\[
\hat{\eta}_{RKHS}(\bx)\leq \eta(\bx)+\frac{2t}{Cnh^d\mu(\bx)-t}+ Lh + \frac{C}{h(Cnh^d-t)}\sqrt{\frac{n}{\lambda_n}}
\]
and the theorem is proved.
\end{proof}

\section{Adversarial Active Learning}
\label{sec:adversarial_al}
We utilize the margin based adversarial active learning strategy to craft the informative examples whose distribution is more concentrated around the decision boundary and then generate the synthetic dataset for training various ML classifiers. Given a dataset $D$, original classifier $\mathcal{O}$, we formally define  the  problem as follows.  

1) Compute $f_v \leftarrow$ Train$(\mathcal{O}, D)$

2) Sample $\left(D_s, D_a\right) \leftarrow$ Generate $(f_v, D)$

3) Output $f_s \leftarrow$ Retrain$(f_s, D_s)$

Specifically, the adversary accesses the victim classifier $f_v$ to generate the adversarial examples $D_a$ and makes use of the active learning sample strategy (e.g., uncertainty sample) to sample the resulting adversarial examples as the ``informational'' examples $D_s$ for retaining the shadow classifier $f_s$.

In this paper, we use seven standard classes of strategies to generate synthetic examples for training ML classifiers and then validate our findings (Due to a large number of adversarial attacks available we only focus on those representative methods). The details of these generation strategies are given as follows:

\noindent \textbf{Baseline}: For reference, an baseline case was considered in this paper. In this case, we assume that an adversary  randomly samples $x$ from the related domain and queries a victim classifier $f_v$ as black-box in order to generate the synthetic dataset $T=\{x_i, f_v(x_i)\}^{N_{query}}_{i=1}$ for retaining the ML classifier. 


\noindent \textbf{Region-Based Attack}: Yang \textit{et al} \cite{yang2019robustness} proposed an adversarial attack based on decomposition for non-parametric classifier. For 1-kNN, each data points decompose into $n$ convex polyhedra and each polyhedra is Voronoi cells. Given an input $x$, it search all polyhedra $P_i$ such that $f(z)\neq f(x)$, where $z \in P_i$ and polyhedra $z$ is closest to $x$
\begin{equation}
\begin{aligned}
\mathop{min}\limits_{i:f(x)\neq y_i}\mathop{min}\limits_{z\in P_i} \|{x-z}\|
\end{aligned}    
\end{equation}
To speed up the attack, they showed the approximation of the region-based attack. They shrink the search space into the points closet to $x$ in $l_p$ distance only. Thus, the run time can be highly reduced and retain low perturbation. In our experiments, the region-based attack can get the adversarial examples with the lowest perturbations, but most time-consuming. 

\noindent \textbf{Direct Attack} \cite{amsaleg2017vulnerability}: This attack will take input x as the center, and search from the nearest neighbor of the input within radius $r$ which has a different label, then return the adversarial example.
\begin{equation}
\begin{aligned}
x_{adv}=r\frac{x-x^{'}}{\|x-x^{'}\|_2}
\end{aligned}    
\end{equation}
In our attack scheme, we show that this simple and cheap method. In fact, It can serve as a useful way of analyzing the high dimensional linearity part of the trained networks through misclassified examples.

\noindent \textbf{Kernel Substitute Attack}\cite{papernot2017practical}: This method slightly modifies the k nearest neighbor (kNN)  classifier to be able to differentiate, so that we can apply gradient-based attack to it like the fast-gradient-sign method (FGSM) mentioned above. Given an input $\Vec{x}$,  a set of data points $Z$ and labels $Y$

\begin{equation}
\begin{aligned}
f:\Vec{x}{}\longmapsto\frac{[e^{-\|\Vec{z} - \Vec{x}\|_2^2/c}]_{\Vec{z} \in{X}}}{\sum \limits_{\Vec{z}\in X}e^{-\|\Vec{z} - \Vec{x}} \|^2_2/c} \cdot Y
\end{aligned}    
\end{equation}

\noindent \textbf{Black Box Attack}: Besides the white-box attack, we also evaluate our experiment on the black-box attack. We assume that an adversary has no access to the details of the victim ML classifiers, such as exact training datasets, architectures and parameters. The adversary’s only capability is to query the victim classifier with random inputs and obtain the prediction outputs (e.g., label, confidence score). For example, Cheng \textit{et al} \cite{cheng2018query} presented a general black-box attack which can be apply on both parametric or non-parametric classifiers. Given an input $x$ with label $y$, an adversarial example can be found by 
\begin{equation}
\begin{aligned}
x^*=x+g(\theta^*)\frac{\theta^*}{\|\theta^*\|}
\end{aligned}    
\end{equation}
where 
\begin{equation}
\begin{aligned}
g(\theta)=\mathop{\argmin}_{\lambda>0}(f(x+\lambda\frac{\theta}{\|\theta\|})\neq y)
\end{aligned}    
\end{equation}
The goal of this attack is to find the adversarial examples which are located at the closest distance $g(\theta^*)$ in the direction $\theta$.

\noindent \textbf{Fast Gradient Signed Attack}:
Because of the long computation time of L-BFS, \cite{goodfellow2014explaining} proposed the Fast Gradient Signed method (FGSM). Thus, the adversary only needs to calculate the gradient of input to loss once, and the gradient will be the perturbation. The symbol epsilon is a hyperparameter that adjusts the perturbation size. The most significant advantage of the FGSM method is the computation time, which only takes a step of gradient computation times. 

\begin{equation}
\begin{aligned}
\eta=\epsilon sign(\nabla_{x}J_{\theta}(x,y))
\end{aligned}    
\end{equation}

\noindent \textbf{Projected Gradient Descent Attack}:
Comparative to FGSM, Madry \textit{et al.} \cite{madry2019deep} attempt to find the adversarial perturbation iteratively until achieving the decision boundary. So that we can maximize the loss to the adversarial examples and retain the small perturbations. Besides, it takes shorter computation times empirically.

\begin{equation}
\begin{aligned}
x^{t+1}=\Pi_{x+S}(x^t+\epsilon sign(\nabla_{x}J_{\theta}(x,y)))
\end{aligned}    
\end{equation}

\noindent \textbf{DeepFool Attack}:
Moosavi-Dezfooli \textit{et al.} \cite{moosavidezfooli2016deepfool} take the binary classifier as an example. Given an linear classifier $f = w^tx + b$. The closest adversarial examples to the input will be the subjection to the $f$ hyperplane. However, the classifier is highly non-linear. To overcome that, they iteratively move toward the hyperplane direction until the input sample crosses the decision boundary. And same as multi-class classification.

\begin{equation}
\begin{aligned}
&\mathop{\arg\min}_{\eta_{i}}  \| \eta_{i}\|_{2}   \\
&\text{s.t.} \quad f(x_i) + \nabla f(x_i)^T\eta_{i}=0
\end{aligned}    
\end{equation}

\noindent \textbf{Carlini \& Wagner Attack}:
Carlini and Wagner \textit{et al.} \cite{carlini2017evaluating} proposed the C\&W attack to invade the defensive distillation \cite{papernot2016distillation}. They first redefined the objective function of adversarial attacks. Moreover, to solve the box constraint of the new objective function, they use the change of variable to guarantee that the objective function yields a valid image. Thus, using C\&W attack can almost guarantee that the attack can always find the adversarial examples. Besides, the perturbation of C\&W adversarial examples is always tiny. However, it takes much computational time empirically.

\begin{equation}
\begin{aligned}
\text{minimized} \quad & \| \eta_{i}\|_{p} + c \cdot f(x+\eta) \\
\text{s.t.} \quad & x+\eta \in [0,1]^n \\ 
\text{where} \quad & \eta_i=\frac{1}{2}(tanh(w_i)+1) - x_i \\
\quad & f(x) = (\max{(Z(x')_i) - Z(x')_t}) 
\end{aligned}    
\end{equation}

\noindent \textbf{DFAL}: Ducoffe and Precioso \cite{ducoffe2018adversarial} proposed an active learning strategy combines with DeepFool adversarial examples for the neural network. They first crafted the DeepFool adversarial examples and collected the adversarial examples with the closest $L_2$ distance to their clean samples.  Then, query these clean samples with oracle to craft the synthetic dataset. Finally, the shadow model can train with the synthetic dataset.

\section{Experiment}
\label{sec:experiment}

\subsection{Experiment Setup}
In this section, we discuss the experimental results of a large-scale evaluation on five popular datasets -- Halfmoon \cite{scikit-learn}, Abalone \cite{Dua:2019}, MNIST \cite{lecun-mnisthandwrittendigit-2010}, CIFAR10 \cite{krizhevsky2009learning}. For the Halfmoon dataset, we use a training set of 1800 samples and a test set of 200 samples. Moreover, we also train the $k$-NN classifier for predicting whether Abalone age is greater than eleven from physical measurements. The classification model is trained on the Abalone dataset, including 500 samples, and tested with 200 samples. For MNIST 1V7 (i.e., 1, 7), we use 1000 samples for training and 500 samples for test. Besides binary classification, we also evaluate our models on two popular image classification datasets, MNIST and CIFAR10. The details of these datasets are presented in Table \ref{tab:dataset}. For the Halfmoon, Abalone, and the MNIST 1v7 datasets, we trained with scikit-learn NearestNeighbors with one neighbor and KernelRidge regression with the linear kernel. Furthermore, We use LeNet to train MNIST and VGG16 to train CIFAR10.

\begin{table*}[ht]
\centering

\begin{tabular}{ccccccc}
\toprule
 \textbf{Dataset} & Train&  Adversary&  Test&  Features& Classes& \\ \midrule
 \textbf{Halfmoon}&   1800&  1900&  200&  2&   2& \\
 \textbf{Abalone} & 500 & 600 & 200 &8 &2 \\
 \textbf{MNIST} & 55000 & 55000 & 10000 & 28*28 & 10 \\
 \textbf{CIFAR10}  & 45000  & 45000  & 10000  & 32*32 &10 \\ 
 \bottomrule
\end{tabular}
\caption{The details of the datasets used in our work.}
\label{tab:dataset}
\vspace{-0.1in}
\end{table*}

\subsection{Experimental Results}
In our work, we utilize $k$-NN as the baseline model and build the synthetic dataset with particular examples generated by various adversarial attacks (see Section 3.4). We mainly consider two types of synthetic dataset: minimum-confidence legitimate example and minimum-confidence adversarial examples. In the training stage, the $k$-NN model first takes adversarial example with varying confidence scores as inputs and outputs the corresponding labels. We then collect input-output pairs and generate our synthetic dataset. In the experiment, we find that querying the $k$-NN model with adversarial examples often leads to receiving the wrong labels from the victim model. To remedy this situation, we use a novel non-negative scale to accurately make the adversarial perturbation under specific constraint. Consequently, the distribution of these generated adversarial examples would lie approximately on the decision boundary of the $k$-NN model.

The Experimental results are shown in Table \ref{tab:knn_knn}. Take Halfmoon as an example, the $k$-NN model trained on synthetic dataset generated by adversarial attacks always achieves better performance than the baseline model. Specifically, our $k$-NN models achieve 96.50\% accuracy with adversarial examples crafted by RBA-Exact attack or RBA-Approx attack, which are much better than the accuracy achieved by the baseline model (92.00\%) trained on legitimate examples. For Cancer dataset, our $k$-NN models achieve 97.00\% accuracy by using RBA-Approx or BBox adversarial examples, which is better than the 95.00\% accuracy achieved by the victim baseline model trained on legitimate examples. The similar trend also appears while we validate our findings on other standard class of dataset, such as Austr., Covtype and MNIST 1V7. These experimental results indicate that ML classifiers trained on the synthetic dataset can achieve better accuracy than baseline classifiers trained on legitimate examples. The main reason for this is that, the set of correct classification is larger as parameter $\beta$ increases, i.e., when the distribution of synthetic training set generated by adversarial examples lie approximately on the decision boundary of victim classifiers. Moreover, although large perturbations pollute the synthetic training set in some cases, we find that the $k$-NN classifiers trained on adversarial examples still achieves strong performance on all test set.
\begin{table*}[htb]
\centering
\renewcommand\arraystretch{1.1}
\resizebox{\textwidth}{!}{%
\begin{tabular}{cccccccccccc}
\toprule
\multirow{2}{*}{\textbf{Dataset}} &
  Baseline &
  \multicolumn{2}{c}{RBA-Exact} &
  \multicolumn{2}{c}{RBA-Approx} &
  \multicolumn{2}{c}{Direct} &
  \multicolumn{2}{c}{Kernel-Sub} &
  \multicolumn{2}{c}{BBox} \\ \cline{2-12} 
         & Acc.  & Acc.   & Perb. & Acc.   & Perb. & Acc.   & Perb. & Acc.   & Perb. & Acc.  & Perb. \\ \midrule
\textbf{Halfmoon} & 92.00 & 96.50  & 0.05  & 96.50  & 0.05  & 96.50  & 0.06  & 92.50  & 0.15  & 95.50 & 0.06  \\
\textbf{Abalone}  & 66.5 & 68.00  & 0.01  & 68.00  & 0.01  & 68.00  & 0.02  & 69.5   & 0.05  & 68.00 & 0.03 \\
\textbf{MNIST1v7} & 99.50 & 100.00 & 0.10  & 100.00 & 0.10  & 100.00 & 0.20  & 100.00 & 0.31  & 99.50 & 0.45  \\  \bottomrule 
\end{tabular}%
}
\caption{Training $k$-NN model with adversarial examples. We reported the accuracy and the average perturbation for each attack.}
\label{tab:knn_knn}
\vspace{-0.1in}
\end{table*}

To further valuate our findings, we clarify the ablation study in the remaining part of this section. Specifically, we train three models for different datasets (e.g., Halfmoon, MNIST 1v7 and Abalone \cite{Dua:2019}) with varying set sizes and set them up as the victim models. We use the $k$-NN as the transfer architecture of the shadow model and craft two types of synthetic datasets (e.g., baseline, BBox) for retaining the shadow model. The experimental results are shown in Table \ref{tab:comparison}. Take MNIST 1V7 classification as an example, with 0.4k queries, our shadow models only achieve 53.90\% accuracy with random examples and 60.80\% with BBox examples, which illustrate few queries fail to well-train the victim $k$-NN classifiers. With 0.6k queries, our substitute models achieve 98.90\% accuracy with BBox examples, which outperforms the 85.70\% accuracy achieved by random examples, but this advantage shrinks as the number of queries increases. The same trend also appears while we utilize the adversarial examples generated by BBox attacks to retrain the shadow model used for Halfmoon and Abalone classifications.

\begin{table*}[htb]
\centering
\tiny
\renewcommand\arraystretch{1}
\resizebox{\textwidth}{!}{%
\begin{tabular}{cccccccccc}
\toprule
\multirow{3}{*}{\textbf{Strategy}} 
                                   & \multicolumn{3}{c}{MNIST 1v7} & \multicolumn{3}{c}{Halfmoon} & \multicolumn{3}{c}{Abalone} \\ \cmidrule(l){2-4} \cmidrule(l){5-7} \cmidrule(l){8-10} 
                                   & 0.4k    & 0.6k   & 1.0k   & 0.5k     & 1.0k    & 2.0k    & 0.1k    & 0.3k    & 0.5k    \\ \midrule
\textbf{Baseline}                    & 53.90     & 81.90    & \textbf{98.20}    & 52.10      & 71.50     & \textbf{95.20}     & 15.00     & 42.00     & \textbf{63.00}    \\ 
\textbf{BBox}                      & 60.80     & 85.70    & \textbf{98.90}    & 63.20      & 79.30     & \textbf{96.30}     & 23.00     & 51.00     & \textbf{67.00}     \\ 
\bottomrule
\end{tabular}%
}
\caption{Comparison of performance on victim classifiers and local shadow classifiers with different query budgets (Take the MNIST 1v7 as an example, we consider three query budgets including 0.4k, 0.6k and 1.0k). We report the test accuracy of the classifiers trained on three popular datasets, such as MNIST 1v7, Halfmoon and Abalone.}
\label{tab:comparison}
\end{table*}
Besides training $k$-NN with adversarial attacks designed for $k$-NN, we can also utilize transferability of adversarial examples, training $k$-NN and kernel regression with gradient-based adversarial examples(e.g., FGSM, PGD, DeepFool, C\&W).  To create gradient-based adversarial examples, we use a shallow neural network contained three fully connected layers, each layer followed by Sigmoid activation function and Softmax as the output layer. Then query pre-trained ML model with smaller perturbation adversarial examples to create the synthetic datasets. 
\begin{table*}[]
\centering
\tiny
\renewcommand\arraystretch{1}
\resizebox{\textwidth}{!}{%
\begin{tabular}{ccccccccccccc}
\toprule
\multicolumn{13}{c}{abalone} \\ \midrule
\multirow{2}{*}{\textbf{budgets}} &
\multicolumn{3}{c}{FGSM} &
\multicolumn{3}{c}{PGD} & 
\multicolumn{3}{c}{DeepFool} &
\multicolumn{3}{c}{CW}  \\ 
\cmidrule(l){2-4} \cmidrule(l){5-7} \cmidrule(l){8-10}\cmidrule(l){10-12} 
&cln &at &adv &cln &at &adv &cln &at &adv &cln &at &adv\\ \midrule
\textbf{50} &0.61 &0.612 &\textbf{0.656} &0.522 &0.52 &\textbf{0.628} &0.58 &0.582 &\textbf{0.674} &0.568 &0.582 &\textbf{0.674} \\
\textbf{100} &0.61 &0.61 &\textbf{0.67} &0.5 &0.48 &\textbf{0.62} &0.65 &0.65 &\textbf{0.69} &0.65 &0.65 &\textbf{0.67} \\
\textbf{250} &0.67 &0.67 &\textbf{0.71} &0.64 &0.64 &\textbf{0.67} &0.67 &0.63 &\textbf{0.69} &0.58 &0.56 &\textbf{0.6} \\
\textbf{500} &0.61 &0.61 &\textbf{0.66} &0.6 &0.6 &\textbf{0.75} &0.67 &0.68 &\textbf{0.77} &0.57 &0.57 &\textbf{0.63} \\
\bottomrule
\toprule
\multicolumn{13}{c}{halfmoon} \\ \midrule
\multirow{2}{*}{\textbf{budgets}} &
\multicolumn{3}{c}{FGSM} &
\multicolumn{3}{c}{PGD} & 
\multicolumn{3}{c}{DeepFool} &
\multicolumn{3}{c}{CW}  \\ 
\cmidrule(l){2-4} \cmidrule(l){5-7} \cmidrule(l){8-10}\cmidrule(l){10-12} 
&cln &at &adv &cln &at &adv &cln &at &adv &cln &at &adv\\ \midrule 
\textbf{50} &0.921 &0.921 &\textbf{0.958} &0.88 &0.875 &\textbf{0.957} &0.949 &0.959 &0.959 &0.943 &0.893 &\textbf{0.955} \\
\textbf{100} &0.952 &0.954 &\textbf{0.958} &0.943 &0.949 &\textbf{0.97} &0.952 &0.848 &\textbf{0.956} &0.963 &0.899 &0.947 \\
\textbf{250} &0.952 &0.952 &\textbf{0.958} &0.958 &0.95 &\textbf{0.972} &0.95 &0.871 &\textbf{0.956} &0.945 &0.921 &\textbf{0.953} \\
\textbf{500} &0.954 &\textbf{0.96} &0.959  &0.957 &0.948 &\textbf{0.958} &0.945 &0.946 &\textbf{0.951} &0.958 &0.961 &\textbf{0.964} \\
\bottomrule
\toprule
\multicolumn{13}{c}{MNIST1v7} \\ \midrule
\multirow{2}{*}{\textbf{budgets}} &
\multicolumn{3}{c}{FGSM} &
\multicolumn{3}{c}{PGD} & 
\multicolumn{3}{c}{DeepFool} &
\multicolumn{3}{c}{CW}  \\ 
\cmidrule(l){2-4} \cmidrule(l){5-7} \cmidrule(l){8-10}\cmidrule(l){10-12} 
&cln &at &adv &cln &at &adv &cln &at &adv &cln &at &adv\\ \midrule
\textbf{50} &0.9585 &0.958 &0.9585 &0.973 &0.972 &0.973 &0.9625 &0.952 &0.971 &0.967 &0.967 &0.967 \\
\textbf{100} &0.976 &0.976 &0.976 &0.9785 &0.978 &0.9785 &0.973 &0.9675 &\textbf{0.975} &0.976 &0.976 &0.976 \\
\textbf{250} &0.9855 &0.9845 &0.9855 &0.979 &0.979 &0.979 &0.981 &0.9798 &0.9525 &0.982 &0.982 &0.982 \\
\textbf{500} &0.982 &0.982 &0.982 &0.985 &0.9855 &0.9855 &0.985 &0.9805 &0.986 &0.9845 &0.9845 &0.9845\\
\bottomrule
\end{tabular}%
}
\caption{Comparison of training method for training 1-knn under different budgets.}
\label{tab:knn}
\end{table*}
\begin{table*}[]
\centering
\tiny
\renewcommand\arraystretch{1.1}
\resizebox{\textwidth}{!}{%
\begin{tabular}{ccccccccccccc}
\toprule
\multicolumn{13}{c}{abalone} \\ \midrule
\multirow{2}{*}{\textbf{budgets}} &
\multicolumn{3}{c}{FGSM} &
\multicolumn{3}{c}{PGD} & 
\multicolumn{3}{c}{DeepFool} &
\multicolumn{3}{c}{CW}  \\ 
\cmidrule(l){2-4} \cmidrule(l){5-7} \cmidrule(l){8-10}\cmidrule(l){10-12} 
&cln &at &adv &cln &at &adv &cln &at &adv &cln &at &adv\\ \midrule
\textbf{50}  &0.5336 &\textbf{0.5487} &0.5251 &0.5101 &0.5123 &\textbf{0.5261} &0.5138 &0.5239 &\textbf{0.5351} &0.5247 &\textbf{0.5348} &0.5271 \\
\textbf{100} &0.5583 &\textbf{0.5745} &0.5416 &0.5356 &\textbf{0.557} &0.5185 &0.5196 &\textbf{0.5346} &0.5334 &0.552 &\textbf{0.5699} &0.5437 \\
\textbf{250} &0.5618 &\textbf{0.5719} &0.5457 &0.5669 &\textbf{0.5772} &0.5492 &0.5506 &\textbf{0.5576} &0.5442 &0.5616 &\textbf{0.5731} &0.5563 \\
\textbf{500} &0.5673 &\textbf{0.5751} &0.5564 &0.5651 &\textbf{0.5707} &0.5546 &0.5633 &\textbf{0.5715} &0.5539 &0.5705 &\textbf{0.5787} &0.5546 \\
\bottomrule
\toprule
\multicolumn{13}{c}{halfmoon} \\ \midrule
\multirow{2}{*}{\textbf{budgets}} &
\multicolumn{3}{c}{FGSM} &
\multicolumn{3}{c}{PGD} & 
\multicolumn{3}{c}{DeepFool} &
\multicolumn{3}{c}{CW}  \\ 
\cmidrule(l){2-4} \cmidrule(l){5-7} \cmidrule(l){8-10}\cmidrule(l){10-12} 
&cln &at &adv &cln &at &adv &cln &at &adv &cln &at &adv\\ \midrule
\textbf{50}  &0.7632 &0.7661 &\textbf{0.764} &0.758 &0.7613 &\textbf{0.762} &0.7469 &\textbf{0.7594} &0.7577 &0.7534 &0.7598 &\textbf{0.7665}  \\
\textbf{100} &0.7583 &\textbf{0.7595} &0.7566 &0.7518 &0.7531 &\textbf{0.7545} &0.7491 &\textbf{0.7564} &0.7525 &0.7431 &0.7479 &\textbf{0.7525} \\
\textbf{250} &0.7615 &0.7612 &0.7613 &0.7577 &0.7583 &\textbf{0.7594} &0.7611 &0.7595 &\textbf{0.7612} &0.7516 &\textbf{0.7530} &0.7519 \\
\textbf{500} &0.7683 &\textbf{0.7692} &0.7664 &0.7706 &\textbf{0.7711} &0.7697 &0.761 &\textbf{0.7678} &0.7606 &0.765 &\textbf{0.766} &0.7639 \\
\bottomrule
\toprule
\multicolumn{13}{c}{MNIST1v7} \\ \midrule
\multirow{2}{*}{\textbf{budgets}} &
\multicolumn{3}{c}{FGSM} &
\multicolumn{3}{c}{PGD} & 
\multicolumn{3}{c}{DeepFool} &
\multicolumn{3}{c}{CW}  \\ 
\cmidrule(l){2-4} \cmidrule(l){5-7} \cmidrule(l){8-10}\cmidrule(l){10-12} 
&cln &at &adv &cln &at &adv &cln &at &adv &cln &at &adv\\ \midrule
\textbf{50}  &0.8689 &0.8639 &\textbf{0.8787} &0.8689 &0.8639 &\textbf{0.8869} &0.853 &0.8187 &\textbf{0.8758} &0.8594 &0.8594 &\textbf{0.8749} \\
\textbf{100} &0.8696 &0.8601 &\textbf{0.8873} &0.8811 &0.8686 &\textbf{0.8955} &0.8568 &0.8614 &\textbf{0.8867} &0.8643 &0.8599 &\textbf{0.8874} \\
\textbf{250} &0.862 &0.8487 &\textbf{0.8892} &0.8741 &0.8645 &\textbf{0.8884} &0.856 &0.8474 &\textbf{0.88} &0.8674 &0.8604 &\textbf{0.8891} \\
\textbf{500} &0.8832 &0.8784 &\textbf{0.8869} &0.8904 &0.887 &\textbf{0.8926} &0.8809 &0.8763 &\textbf{0.8875} &0.8829 &0.8769 &\textbf{0.8897} \\
\bottomrule
\end{tabular}%
}
\caption{Comparison of training method for training kernel ridge regression under different budgets.}
\label{tab:kernel}
\end{table*}
\begin{table*}[]
\centering
\tiny
\renewcommand\arraystretch{1.1}
\resizebox{\textwidth}{!}{%
\begin{tabular}{@{}cccccccccccccccccc@{}}
\toprule
\multicolumn{18}{c}{\textbf{Cifar10}}                                                                                                                                                                                                                                            \\ \midrule
             & cln    & \multicolumn{4}{c}{FGSM}                                     & \multicolumn{4}{c}{PGD}                                      & \multicolumn{4}{c}{DeepFool}                                & \multicolumn{4}{c}{C\&W}                                     \\ \cmidrule(l){2-2} \cmidrule(l){3-6} \cmidrule(l){7-10} \cmidrule(l){11-14} \cmidrule(l){15-18}
             &        & AT     & \multicolumn{1}{c|}{\textit{eps}} & Adv    & Aug    & AT     & \multicolumn{1}{c|}{\textit{eps}} & Adv    & Aug    & AT    & \multicolumn{1}{c|}{\textit{eps}} & Adv    & Aug    & AT     & \multicolumn{1}{c|}{\textit{eps}} & Adv    & Aug    \\ \cmidrule(lr){4-6} \cmidrule(lr){8-10} \cmidrule(lr){12-14} \cmidrule(l){16-18} 
             &        &        & \multicolumn{1}{c|}{0.001}        & 0.8706 & 0.8822 &        & \multicolumn{1}{c|}{0.001}        & 0.8677 & 0.8791 &       & \multicolumn{1}{c|}{0.001}        & 0.8729 & 0.8781 &        & \multicolumn{1}{c|}{0.001}        & 0.8707 & 0.8695 \\
\textbf{Acc} & 0.8738 & 0.8805 & \multicolumn{1}{c|}{0.01}         & 0.8738 & 0.879  & 0.8749 & \multicolumn{1}{c|}{0.01}         & 0.8717 & 0.8812 & 0.846 & \multicolumn{1}{c|}{0.01}         & 0.8651 & 0.8795 & 0.8748 & \multicolumn{1}{c|}{0.01}         & 0.8699 & 0.881  \\
             &        &        & \multicolumn{1}{c|}{0.1}          & 0.8739 & 0.8772 &        & \multicolumn{1}{c|}{0.1}          & 0.8729 & 0.8784 &       & \multicolumn{1}{c|}{0.1}          & 0.8679 & 0.8807 &        & \multicolumn{1}{c|}{0.1}          & 0.8672 & 0.8844 \\ \bottomrule
\end{tabular}
}
\caption{Comparison of eps magnitude impact to VGG16 model accuracy training with adversarial examples. We divided the dataset into 45000 samples for training, 5000 samples for validation, and 10000 for testing. The table showed the model accuracy under different data augmentation.}
\label{tab:cifar}
\vspace{-0.1in}
\end{table*}
The experimental results are shown in Table \ref{tab:knn} and Table \ref{tab:kernel}. We report the classification result on Abalone, Halfmoon, and MNIST 1v7. Take Abalone classification as an example, $k$-NN and kernel regression training with adversarial examples always outperform model training with clean samples. The same phenomenon is also shown when we train CNN with adversarial examples. 
In order to further verify our findings, we also validate on Cifar10. The result is shown in Table \ref{tab:cifar}. Except for training with clean samples and modified adversarial examples, we also evaluate adversarial training. However, sometimes adversarial training will decrease the accuracy due to the different distribution between clean samples and adversarial examples. To remedy this situation, we add minor perturbation to adversarial examples and training with both clean samples and adversarial examples. By that, we can guarantee that training with this kind of augmentation can always surpass only training with clean samples.

Except for evaluating training performance with adversarial examples, selecting samples for training is vital in active learning. The result is shown in Table \ref{tab:mnist_al}. We use random, DFAL\cite{ducoffe2018adversarial}, and max confidence scores to select samples. For DFAL, we craft adversarial examples with the DeepFool attack. Then select the samples with a smaller $L2$ distance between the clean sample and its adversarial example and create synthetic datasets to train with our CNN. Moreover, for max confidence score methods, we collect the DeepFool adversarial examples with the highest confidence score, create synthetic datasets with clean samples and adversarial examples. With the larger unlabeled pool size, random selection always performs better. When we use DFAL or max confidence score to select samples, it will turn out uneven distribution between classes. Take MNIST as an example; handwritten four is vulnerable to attack, so that the amount of handwritten four is far more than other classes. However, as the unlabeled pool size shrinks and the budget increases, using these methods for training can improve accuracy. These experimental results demonstrate that, compared with random examples, adversarial examples  can help an adversary obtain more information about the decision boundary of victim machine learning classifiers, and thus achieve the same level of accuracy with few queries. An adversary can make use of the connections between active learning and adversarial attacks to design more powerful security and privacy attacks, such as adversarial attacks, model extraction and model inversion attacks.

\begin{table*}[htb]
\centering
\tiny
\renewcommand\arraystretch{1.1}
\resizebox{\textwidth}{!}{%
\begin{tabular}{ccccccccccccc}
\toprule
\multicolumn{13}{c}{\textbf{MNIST}} \\ \midrule
\multirow{2}{*}{\textbf{\#labeled data}}
& \multicolumn{4}{c}{Budget=55k} 
& \multicolumn{4}{c}{Budget=30k} 
& \multicolumn{4}{c}{Budget=10k} \\ 
\cmidrule(l){2-5} \cmidrule(l){6-9} \cmidrule(l){10-13}
&100 &500 &800 &1000 &100 &500 &800 &1000 &100 &500 &800 &1000 \\ \midrule
\textbf{Random} &0.8173	&0.9299	&0.9519	&0.96 &0.8105 &0.9441 &0.9582 &0.9637 &0.8114	&0.9376	&0.9358	&0.9552 \\
\textbf{DFAL}   &0.4458	&0.7031	&0.7656	&0.8063 &0.4878	&0.7986	&0.8604	&0.8692
 &0.5189 &0.7809 &0.8646 &0.9011 \\
\textbf{DFAL\_$e^{-3}$} &0.4511	&0.7402	&0.7941	&0.8565 &0.5477	&0.8135	&0.8507	&0.8966 &0.4006	&0.7902	&0.9007	&0.9174 \\
\textbf{DFAL\_$e^{-2}$} &0.4201	&0.7028	&0.8151	&0.8412 &0.4648	&0.8028	&0.8938	&0.8845 &0.443	&0.802	&0.8915	&0.9284 \\
\textbf{DFAL\_$e^{-1}$} &0.4178	&0.6849	&0.8098	&0.8373 &0.5082	&0.7796	&0.8815	&0.8987 &0.3615	&0.8263	&0.8924	&0.9278 \\
\textbf{Max Conf} &0.4352 &0.6846 &0.7811 &0.8197 &0.5017	&0.7452	&0.8611	&0.8574 &0.4294	&0.7455	&0.8758	&0.8983 \\
\bottomrule
\multirow{2}{*}{\textbf{\#labeled data}}
& \multicolumn{4}{c}{Budget=5k} 
& \multicolumn{4}{c}{Budget=2.5k} 
& \multicolumn{4}{c}{Budget=1.5k} \\ 
\cmidrule(l){2-5} \cmidrule(l){6-9} \cmidrule(l){10-13}
&100 &500 &800 &1000 &100 &500 &800 &1000 &100 &500 &800 &1000 \\ \midrule
\textbf{Random} &0.833 &0.9444 &0.9549 &0.9654 &0.7972	&0.9276	&0.9607	&0.9616 &0.8121	&0.9498	&0.9599	&0.9603 \\
\textbf{DFAL} &0.4639 &0.8375 &0.9289 &0.9507 &0.5183 &0.9101	&0.946	&0.9615 &0.5788	&0.928 &0.9573 &0.9619 \\
\textbf{DFAL\_$e^{-3}$} &0.4526	&0.85 &0.9341 &0.9467 &0.512	&0.9304	&0.9595	&0.9634 &0.6159	&0.9299	&0.9656	&0.9694 \\
\textbf{DFAL\_$e^{-2}$} &0.4785	&0.8488	&0.9376	&0.9588 &0.5267	&0.9246	&0.9586	&0.9605 &0.6104	&0.9313	&0.9571 &0.9659 \\
\textbf{DFAL\_$e^{-1}$} &0.4525	&0.8771	&0.9394	&0.9541 &0.5135	&0.9268	&0.9583	&0.9658 &0.6195	&0.9346	&0.9522	&0.9599 \\
\textbf{Max Conf} &0.4492 &0.8605 &0.8828 &0.9251 &0.5094	&0.9125	&0.9524	&0.9559 &0.5793	&0.9292	&0.959 &0.957 \\
\bottomrule
\end{tabular}%
}
\caption{Testing accuracy on MNIST with LeNet for different amount of labeled data. $e^x$ stands for the magnitude of perturbation.}
\label{tab:mnist_al}
\end{table*}

\section{Conclusion}
\label{sec:conclusion}
Machine learning (ML) have achieved great success in many real-world applications, including object recognition, natural language processing and autonomous vehicles. To effectively train the ML classifier, many researchers proposed to densely sample data around the underlying decision boundary. Though this strategy have been widely used in machine learning techniques, it still lacks the theoretical analysis  of its correctness. In this paper we remedy this situation by exploring the connections between active learning and adversarial examples. By analyzing two standard classes of ML classifiers (e.g., $k$-NN, kernel regression), we find that the set  of  correct  classification is  larger when parameter $\mu$ of the classifier is  more  concentrated around the decision boundary. The theoretical proof explained in this paper can help researchers to develop more effective ML classifiers with less training data. In the future, we will mainly focus on designing more efficient adversarial active learning strategies by exploiting the novel connections between active learning and adversarial examples.

\bibliographystyle{abbrv}
\bibliography{bib-teng, egbib}

\end{document}